\documentclass[a4paper,11pt]{article} 

\title{Online Optimization of Smoothed Piecewise Constant Functions}

\usepackage{dsfont}
\usepackage{amsmath}
\usepackage{amssymb}
\usepackage{algorithmicx}
\usepackage{algorithm}
\usepackage{algpseudocode}
\usepackage{graphicx}
\usepackage{url}
\usepackage{hyperref}
\usepackage{subcaption}
\usepackage{caption}
\usepackage{amsthm}
\usepackage[round]{natbib}
\bibpunct{(}{)}{;}{a}{,}{,}
\usepackage{fullpage}
\theoremstyle{plain}
\newtheorem{theorem}{Theorem}[section]
\newtheorem{lemma}[theorem]{Lemma}

\newtheorem{observation}[theorem]{Observation}

\newtheorem{corollary}[theorem]{Corollary}

\newtheorem{claim}[theorem]{Claim}

\newcommand{\opt}{\text{OPT}}

\newcommand{\calI}{\mathcal{I}}
\newcommand{\calT}{\mathcal{T}}

\DeclareMathOperator*{\argmin}{arg\!min}

\usepackage{color}
\usepackage{xcolor}

\def\algfor{ {\bf for}}
\def\algdo{ {\bf do}}
\def\algset{ {\bf set}}
\def\alginput{ {\bf input}}
\def\poly{\mathrm{poly}}

\def\etc{ {\it etc.}}
\def\eg{ {\it e.g.,}~}
\def\ie{ {\it i.e.,}~}

\def\E{\mathbb{E}}
\def\F{{\mathcal F}}
\def\calH{{\mathcal H}}
\def\naturals{\mathbb{N}}
\def\reals{\mathbb{R}}

\def\Alg{\mathrm{Alg}}
\def\Regret{\mathrm{Regret}}

\newif\ifshort
\shorttrue 

\author{ }
\author{Vincent Cohen-Addad\thanks{\'{E}cole normale sup\'{e}rieure. Email: vincent.cohen@ens.fr} \and Varun Kanade\thanks{University of Oxford. Email: varunk@cs.ox.ac.uk}}

\begin{document}


\maketitle

\begin{abstract}
	We study online optimization of \emph{smoothed} piecewise constant functions
	over the domain $[0, 1)$. This is motivated by the problem of adaptively
	picking parameters of learning algorithms as in the recently introduced
	framework by \cite{GR2016}. 
	Majority of the machine learning literature has focused on
	Lipschitz-continuous functions or functions with bounded
	gradients.\footnote{These functions are typically called \emph{smooth} in
	the machine learning literature. We avoid this usage here, since we use
	\emph{smoothed} and \emph{smoothness} in the sense of \cite{ST:2004}.} 
	This is with good reason---any learning algorithm suffers linear regret even
	against piecewise constant functions that are chosen adversarially, arguably
	the simplest of non-Lipschitz continuous functions. The \emph{smoothed}
	setting we consider is inspired by the seminal work of \cite{ST:2004} 
	and the recent work of \cite{GR2016}---in 
	this setting, the sequence of functions may be chosen by an adversary,
	however, with some uncertainty in the location of discontinuities. We give
	algorithms that achieve sublinear regret in the full information and bandit
	settings.
\end{abstract}

\section{Introduction}

In this paper, we study the problem of online optimization of piecewise
constant functions. This is motivated by the question of selecting
\emph{optimal parameters} for learning algorithms. Recently, \cite{GR2016}
introduced a \emph{probably approximately correct} (PAC) framework for choosing
parameters of algorithms. Imagine a situation, when a website wishes to provide
personalized results to a user. To respond to a user's query, the service
provider may need to implement a learning (or some other type of) algorithm
which involves choosing parameters. The choice of parameters affects the
quality of solution and ideally we would like to design a mechanism where the
service provider learns from past instances, or at least employs a strategy
that has low regret with respect to the \emph{single} optimal solution in
hindsight. In many learning problems, the goal is to find parameters by
optimizing a continuous function (of the parameters); however, ever so often
one encounters problems with discrete solutions, such as $k$-means or
independent set, which result in objective functions that have discontinuities.
 
Concretely, we consider the problem of online optimization of piecewise
constant functions over the domain $[0, 1)$. At each round the learning
algorithm plays a point $x_t \in [0, 1)$, receives payoff $f_t(x_t)$, where
$f_t$ is a piecewise constant function. The aim of the learning algorithm is to
achieve \emph{no-regret} with respect to the best single point $x^* \in [0, 1)$
in hindsight.\footnote{Unfortunately, the confusing terminology
\emph{no-regret} is common in the literature and is used to denote that the
regret grows sublinearly in $T$, the number of rounds for which the online
algorithm is run.} As is standard, by \emph{no regret} we mean, regret that
grows sub-linearly with $T$, the number of rounds played.  If we make no
assumptions about how the piecewise linear functions $f_t$ are chosen then, it
is easy to construct instances where the algorithm would suffer regret that is
linear in $T$.

We take the view that real-world problems, while not entirely
stochastic are rarely truly adversarial.\footnote{There may be settings where
assumption of a malicious adversary is justified, \eg when spammers, google
bombers, \etc~are actively seeking to compromise systems.} In this work, we
consider a \emph{smoothed} adversary;  rather than defining a piecewise
constant function $f$ over $[0, 1)$ by defining the intervals $[0, a_1),
\ldots [a_{k- 1}, 1)$ exactly, the adversary may only define distributions to
pick the points $a_i$, with the added constraint that the density of these
distributions is upper bounded by some parameter $\sigma$. It is very natural
to assume that there is uncertainty in defining real-world problems, either due
to noise or imperfect information; indeed this was also the motivation of the
original work by \cite{ST:2004} 
where they showed that the smoothed time complexity of (a variant of) the
simplex algorithm is polynomial.  This uncertainty in defining the intervals
(or the points of discontinuity) is precisely what we exploit in designing
no-regret algorithms.

Machine learning research has primarily focused on optimizing
functions with bounded (first few) derivatives.  For example, there exists
substantial literature on online optimization of Lipschitz continuous
functions, both in the full information and the bandit setting (see \eg
\citep{Kle:2004,KSU:2008,BMG:2009}). However, any sort of combinatorial
structure typically introduces discontinuities in the objective function. Thus,
most of the existing methods for online optimization are no longer applicable.

The \emph{smoothness} formulation we use in the paper restricts an adversary
from being able to define too narrow an interval in which \emph{optimal
solutions} may lie. In particular, if we consider the refinement of all the
intervals we get over $T$ rounds of the (smoothed) adversary choosing piecewise
constant functions, the smallest interval is still polynomially small in $T$
(and the bound $\sigma$ on the density and the number of pieces $k$). This
ensures that the problem is not that of finding a needle in a haystack, but a
rather hefty iron rod. Under these conditions, in principle, one could simply
draw a large enough (but still polynomial in $T$) number of points uniformly in
the interval $[0, 1)$, and consider the problem as the standard experts
setting. The bandit setting is a bit more delicate, but still could be handled
using ideas similar to the $\mathrm{Exp.4}$ algorithm \citep{AC-BFS:2003}.  The
difficulty is \emph{computational}---we would rather design algorithms that at
time step $t$, run in time polynomial in $\log(t)$ and other problem-dependent
factors, than in time polynomial in $t$.  With carefully designed algorithms
and data-structures, we can indeed achieve this goal. We summarize our
contributions below, describe related work, and then discuss how our work fits
in a broader context.

\subsection{Our Results}

We show that against a \emph{smoothed} adversary, one can design algorithms
that achieve the almost optimal regret of $\tilde{O}(\sqrt{T})$ (the
$\tilde{O}(\cdot)$ notation hides poly-logarithmic factors) in the expert
setting, \ie when we observe the entire function $f_t$ at the end of the round.
Our algorithm is based on a continuous version of the exponentially-weighted
forecaster. A na\"{i}ve implementation of the algorithm we propose would result
in a running time that grows polynomially in $t$. We design a data structure
based on \emph{interval trees} or \emph{red-black trees} (see Section
\ref{sec:interval-trees} for the full description) that allows us to implement
our algorithm extremely efficiently---the running time of our algorithm is
$O(k\log(kt))$ at any given timestep, where $k$ is the number of pieces of the
piecewise constant function at each round. We remark that at least logarithmic
dependence in $t$ is required in the sense that even keeping track of time
requires time at least $\log(t)$. 

We also consider the \emph{bandit setting}: here the algorithm does not observe
the entire function $f_t$, but only the value $f_t(x_t)$ for the point chosen.
In order to estimate the function $f_t$ elsewhere, we optimistically assume
that it is constant in some (suitably chosen) small interval around $x_t$. Of
course, sometimes the point $x_t$ chosen by the algorithm may lie very close to
a point of discontinuity of the function $f_t$. However, this cannot happen
very often because of the \emph{smoothness} constraint on the adversary. With a
somewhat delicate analysis we obtain a regret bound of $\tilde{O}(\poly(k,
\sigma) T^{2/3})$ in the bandit setting. As in the full information (or expert)
setting, our algorithm is very efficient, \ie polynomial in $\log(t)$ and other
parameters such as $k$ and $\sigma$.

In Section~\ref{sec:experiments}, we consider a few different problems in this
setting---knapsack, weighted independent set, and weighted $k$-means. The first
two of these problems were already considered by \cite{GR2016}, 
however they do not provide \emph{true} regret bounds in their paper, \ie
bounds where the average regret approaches $0$ as $T \rightarrow
\infty$.\footnote{We remark that their algorithm will be a ``true'' no-regret
algorithm with suitably chosen parameters, but in that regime the running time
is polynomial in $T$.} The combinatorial nature of these problems means that as
we vary the algorithm parameters, there may be discontinuities in the objective
function. Our preliminary experiments indicate that the behavior does indeed
correspond to that predicted by theoretical bounds, albeit with slightly better
rates of convergence (possibly since we generate instances randomly).

\subsection{Related Work}

The work most closely related to ours is that of \cite{GR2016}. 
In the context of online learning, our work improves upon theirs in providing
bounds for online learning of algorithm parameters that are \emph{true regret}
bounds; in their paper they only provide $\epsilon$-regret bounds, in that one
can guarantee that for any given $\epsilon$ the algorithm will achieve average
regret of $\epsilon$. The algorithms we present in this paper are more natural
and achieve a significant improvement in running time. We also give results in
the \emph{bandit} setting, which is in many ways more appropriate for the
applications under consideration. The approach considered in their paper does
not yield a bandit algorithm. With some effort, one may be able to adapt ideas
from the $\mathrm{Exp}.4$ algorithm of \cite{AC-BFS:2003} to achieve a
non-trivial regret bound in the \emph{bandit} case; however, the resulting
algorithm would be computationally expensive.

There is a substantial body of work that seeks to use learning mechanisms to
choose the parameters or \emph{hyperparameters} of algorithms.
\cite{SSZA:2014, SLA:2014} suggest using Bayesian optimization
techniques to choose hyperparameters effectively. Yet other papers (see \eg
\citep{F:98, HJYCMN:10, KGM:12, HXHL:15}) suggest various techniques to choose
parameters for algorithms (not necessarily in the context of learning).
However, except for the work of \cite{GR2016}, most work is not theoretical in
nature.

\subsection{Discussion}
\label{sec:discussion}

The notion of \emph{smoothness} considered in this paper is inspired by the
seminal work of \cite{ST:2004}. 
In theoretical computer science, this notion allows us to look beyond
worst-case analysis without making extremely strong assumptions required for
average-case analyses. This approach seems particularly relevant to machine
learning, a filed in which worst-case results or those using strong
distributional assumptions typically have little bearing in practice.  The
smoothness considered here is on the instances themselves rather than on the
functions being optimized as is common in machine learning.  Combinatorial
problems arise naturally in several online and offline learning settings; in
these cases the notion of \emph{smoothness} \`a la Spielman and Teng may be
more appropriate than the traditional notion of Lipschitz continuity. It would
be interesting to explore if this notion of smoothness is applicable in other
settings, \eg sleeping combinatorial experts and bandit settings---where it is
known that \emph{stochastic instances} are often tractable, while adversarial
ones are \emph{computationally hard} (see \eg \citep{KS14, NV14, KLP15}).

A natural open question is whether our work could be extended to more general
functions with discontinuities, such as piecewise linear or piecewise Lipschitz
functions. If computational cost were not a concern, we believe this could be
achieved (at least in the full-information setting) by choosing a fine enough
grid of $[0, 1)$ as experts. However, whether efficient algorithms
such as ours for the case of piecewise constant functions can be designed is
an open question.

\section{Setting}

We consider the online optimization setting where the decision space is $[0,
1)$. At each time step $t$, the learning algorithm must pick a point $x_t \in
[0, 1)$ to play. A \emph{smoothed oblivious adversary} picks a function $f_t :
[0, 1) \rightarrow [0, 1]$ that is piecewise constant as follows:
\begin{enumerate}
	\item Adversary defines distributions $D_{t, 1}, \ldots, D_{t, k-1}$ where the support of each $D_{t, i}$ is contained in $(0, 1)$ and the density functions are bounded by $\sigma$.
	\item Adversary defines values $v_{t,1}, \ldots, v_{t,k} \in [0, 1]$
	\item Nature draws $a^\prime_{t, i} \sim D_{t, i}$ independently for $i = 1,
	\ldots, k-1$.  Let $0 = a_{t, 0}, a_{t, 1}, \ldots, a_{t, k-1}, a_{t, k} =
	1$ be in non-decreasing order, where $a_{t,1}, \ldots, a_{t, k-1}$ are just
	$a^\prime_{t, 1}, \ldots, a^\prime_{t, k-1}$ sorted.\footnote{Under the
	smoothness assumptions, the $a^\prime_{t, i}$ will be distinct with
	probability $1$.}
\item The piecewise constant function $f_t$ is defined as $f_t(x) = v_{t,i}$ for
	$x \in [a_{t, i - 1}, a_{t, i})$.
\end{enumerate}

For a known time horizon $T$, let $x_1, \ldots, x_T$ be the choices made by the
learning algorithm.\footnote{We assume that the time horizon $T$ is known,
otherwise, the standard doubling trick may be applied.} Then the regret is
defined as: 
\begin{align*}
	\Regret(\Alg) = \max_{x \in [0, 1)} \sum_{t = 1}^T f_t(x) - \sum_{t = 1}^T f_t(x_t)
\end{align*}

We consider the full information (or \emph{experts}) setting, where at the end
of each round the full function $f_t$ is revealed to the learning algorithm. We
also look at the \emph{bandit} setting, where the learning algorithm only sees
the value $f_t(x_t)$. All results in this paper are stated in terms of expected
regret; we believe that bounds that hold with high probability can be obtained
using standard techniques.

All the results in this paper can easily be generalized to the setting where
the decision space is $[0, 1]^d$ and the adversary chooses functions that are
constant on sub-hypercubes. In the full-information setting, the regret
guarantees will be worse by a factor that is polynomial in $d$ and the running
time worse by a factor exponential in $d$. In the bandit setting, both the
regret and the running time will be worse by a factor exponential in $d$. For
simplicity we only discuss the one dimensional setting and defer the general
case to the full version of the paper.

We first state some useful observation that we will use repeatedly in this
paper.  These are by no means original and already appear in some form in
\citep{GR2016} for example.

\begin{observation} \label{obs:even-spread}
Let $x_1, \ldots, x_m$ be independently drawn from any distributions (possibly different) whose density functions are bounded by $\sigma$. Then the probability that there exist $x_i < x_j$ such that $x_j - x_i < \epsilon$ is at most $m^2 \sigma \epsilon$.
\end{observation}
\begin{proof}
The proof is just an application of the union bound over all $m \choose 2$ pairs. 
\end{proof}

\begin{observation} \label{obs:fine-net}
Let $0 = x_0 < x_1 < \cdots < x_m = 1$ be any points in $[0, 1]$ such that $x_{i} - x_{i - 1} > \epsilon$ for all $i$. Let $y_1, \ldots, y_N$ be drawn uniformly at random from $[0, 1)$. Then if $N \geq \frac{1}{\epsilon} (\ln(\epsilon^{-1}) + \ln(\delta^{-1}))$, the probability that there exists $i$ such that there is no $y_j$ in the interval $(x_{i-1}, x_{i})$ is at most $\delta$.
\end{observation}
\begin{proof} This is basically a balls into bins argument with at most $\epsilon^{-1}$ bins and the probability of throwing a ball into each bin is at least $\epsilon$.
\end{proof}

%
%

\section{Full Information Setting}

First in Section~\ref{sec:lower-bound}, we explain why it is necessary to look
at \emph{smoothed adversaries}; without this, one can easily construct a
relatively benign instance that suffers a regret of $\Omega(T)$. Then, we give
an exponentially-weighted forecaster that achieves expected regret that is
$O(\sqrt{\log(T) T})$. We show that in fact this can be implemented very
efficiently; at time $t$ the running time of our algorithm is polynomial in $k$
and $\log(t)$, where $k$ is the number of pieces of the functions.

\subsection{Lower Bound for Worst-Case Adversaries}
\label{sec:lower-bound}
Here, we show that unless one adds a smoothness assumption the worst-case
regret bound is linear in $T$. This is unsurprising and in a way already
follows from Theorem~4.2 in the extended version of \citep{GR2016}. However, in
the more general setting considered in this paper, the bad instance is simpler
to describe and we provide it for completeness.

We describe a simple adversary that chooses functions that are piecewise
constant with at most $3$ pieces and each piece being of length at least $1/5$;
but for allowing the adversary to choose the points of discontinuity arbitrarily, the
freedom given to the adversary is limited. For round $1$ the adversary
chooses $f_1$ such that $f_1(x) = 1$ for $x \in [2/5, 3/5)$ and $0$ otherwise.
On round $2$ the adversary chooses $f_2$ to be $1$ on either the interval
$[3/10, 1/2)$ or $[1/2, 7/10)$ uniformly at random. Thus, any learning
algorithm can get expected payoff at most $1/2$, but either the entire interval
$[2/5, 1/2)$ or $[1/2, 3/5)$ would have got payoff of $1$ on both rounds. In
general if after $t$ rounds, if there is an interval $[l_t, u_t)$ such that
$f_s(x) = 1$ for all $x \in [l_t, u_t)$ for all $s \leq t$, then if $m_t = (l_t
+ u_t)/2$, the adversary picks $f_{t +1}$ to take value on either $[m_t - 1/5,
m_t)$ or $[m_t, m_t + 1/5)$. It is clear that after $T$ rounds the expected
payoff of any learning algorithm is at most $T/2$, however there exists a point
$x^* \in [0, 1)$ that would receive a payoff of $T$.

\subsection{No-Regret Algorithm for Smoothed Adversaries}

Algorithm~\ref{alg:expwt} is a fairly standard exponentially-weighted
forecaster that is used to make predictions in the non-stochastic setting. We
describe the efficient implementation using \emph{modified} interval trees in
Section~\ref{sec:interval-trees}. If one were merely concerned with achieving a
running time that was polynomial in $T$ at each round, there is a rather easy
solution. Using Observation~\ref{obs:even-spread} we know that even after a
full refinement of intervals that appear in the functions $f_1, \ldots, f_T$
with high probability the smallest interval is of length at least $1/T^3$. Then
Observation~\ref{obs:fine-net} tells us that by picking $T^3 \log(T)$ points
uniformly at random, we get a set of points that would hit every interval in
the refinement. Thus, we could pick these points to begin with and apply a
standard expert algorithm, such as Hedge~\citep{FS:1995}. Since the regret
depends only logarithmically on the number of experts, we would still achieve a
regret bound of $O(\sqrt{ \log(T) T })$. In a way, this is what \cite{GR2016}
do to obtain their ``low-regret'' bound. However, this solution is inelegant
and suffers significantly in terms of computational cost. Our algorithm runs in
time polynomial in $k$ and $\log(t)$ at time $t$; note that dependence on
$\log(t)$ is required even just to keep track of time!

\begin{algorithm}
	\alginput: $\eta$ \smallskip \\
	\algset~$F_1(x) = 0$ to be the constant $0$ function over $[0, 1)$ \smallskip \\
	\algfor~ $t = 1, 2, \ldots, T$ \algdo
	\begin{enumerate}
		\item Define $p_t(x) = \frac{\exp(\eta F_t(x))}{\int_{0}^{1} \exp(\eta F_t(x)) dx}$ \label{alg:step:distrib}
		\item Pick $x_t \sim p_t$ \label{alg:step:select}
		\item Observe function $f_t$ and receive payoff $f_t(x_t)$ 
		\item Set $F_{t + 1} = F_t + f_t$  \label{alg:step:update}
	\end{enumerate}
	\caption{No-regret algorithm \label{alg:expwt}}
\end{algorithm}

\begin{theorem}
  \label{thm:noregret}
  With probability $1$, the expected (the expectation is only with respect to
  the random choices of the algorithm, but not those of nature) regret of
  Algorithm~\ref{alg:expwt} is bounded by 
	\[ \eta (e - 2) T - \frac{\ln(\epsilon^*)}{\eta} \]
	where if $x^* \in [0, 1)$ is the best point in hindsight such that $x^* \neq
	a_{t, j}$ for any $t \in [T]$ and $j \in [k]$, \ie $x^*$ is not a point of
	discontinuity for any of the functions $f_t$, 
	\[ \epsilon^* = \min \{a_{t,j} ~|~ t \in [T], j \in [k], a_{t,j} > x^*\} -
	\max \{a_{t,j} ~|~ t \in [T], j \in [k], a_{t,j} < x^*\}, \]
	Here $\epsilon^*$ is the length of the largest interval containing the point
	$x^*$ for which $F_{T + 1}$ is constant (and optimal).
\end{theorem}
\begin{proof}
  Let $x^* \in \argmin F_{T + 1}(x)$ such that $x^* \not\in \{a_{t,j} ~|~ t \in
  [T], j \in [k] \}$. Note that such an $x^*$ exists as long as all the
    functions have distinct points of discontinuity, which happens with probability
    $1$ under our model of max density assumption. Also let $a^* = \min \{a_{t,j}
  ~|~ a_{t,j} > x^*\}$ and $b^* = \max \{ a_{t,j} ~|~ a_{t,j} < x^*\}$. Then, note
  that $F_{T+1}$ is constant on the interval $(a^*, b^*)$. 

  Let $W_t = \int_{0}^{1} \exp(\eta F_{t}(x)) dx$. We follow the fairly
  standard proof of obtaining regret bounds for exponentially-weighted
  forecasting. Let $P_t = \E_{x \sim p_t}[f_t(x)]$ denote the expected payoff
  achieved by the algorithm in round $t$, where the expectation is only with
  respect to the algorithm's random choices. The first observation is to upper
  bound $W_{t + 1}/W_{t}$ by $\exp((e^{\eta} - 1) P_t)$. To see this,
  \begin{align*}
    \frac{W_{t+1}}{W_{t}} &= \frac{\int_0^1 \exp(\eta(F_{t+1}(x)) dx}{\int_0^1 \exp(\eta(F_{t}(x)) dx} \\
    &= \frac{\int_0^1 \exp(\eta F_{t}(x)) e^{\eta f_t(x)} dx}{\int_0^1 \exp(\eta(F_{t}(x)) dx} & \mbox{Since $F_{t + 1} = F_{t} + f_t$} \\
    &= \int_0^1 p_{t}(x) e^{\eta f_t(x)} dx & \mbox{By definition of $p_t$} \\
    &\leq \int_0^1 p_{t}(x) (1 + (e^\eta - 1) f_t(x)) dx & \mbox{For $z \in [0,1]$, $e^{\eta z} \leq 1 + (e^\eta - 1)z$} \\
    &\leq 1 + (e^\eta - 1) P_t \leq \exp((e^\eta - 1) P_t) & \mbox{Since $1 + z \leq e^z$}
  \end{align*}
  
Thus, we get that 
\begin{align}
	\frac{W_{T + 1}}{W_1} &\leq \exp\left( (e^\eta - 1) \sum_{t = 1}^T P_t \right) = \exp((e^\eta - 1) P(\Alg))  \label{eqn:upperbound}
\end{align}
Where $P(\Alg) = \sum_{t = 1}^T P_t$ is the expected payoff of the algorithm (with respect to its random choices).

On the other hand $W_1 = \int_0^1 \exp(\eta F_1(x)) dx = 1$ and,
\begin{align}
	W_{T + 1} &= \int_0^1 \exp(\eta F_{T + 1}(x)) dx \\
	&\geq \int_{a^*}^{b^*} \exp(\eta F_{T+1}(x)) dx & \mbox{Since $\exp(\eta F_{T + 1}(x)) \geq 0$ on $[0, 1)$} \\
	&= \epsilon^* \exp(\eta \opt) & \mbox{where $\opt = F_{T + 1}(x^*)$} \\
	\intertext{Thence, we have,}
	\frac{W_{T+1}}{W_1} &\geq \epsilon^* \exp(\eta\opt) \label{eqn:lowerbound}
\end{align}

Using (\ref{eqn:upperbound}) and (\ref{eqn:lowerbound}) together, after taking logarithms of both sides, we get
\begin{align*}
	\eta \opt + \log(\epsilon^*) &\leq (e^\eta - 1) P(\Alg) \nonumber 
	\intertext{Rearranging terms and dividing by $\eta$, we get}
	\opt - P(\Alg) &\leq \frac{e^\eta -1 - \eta}{\eta} P(\Alg) - \frac{\log(\epsilon^*)}{\eta} \nonumber
	\intertext{For $\eta \in [0, 1]$, $e^\eta \leq 1 + \eta + (e - 2)\eta^2$ and $P(\Alg) \leq
	T$ since the range of $f_t$ is $[0, 1]$. Thus, we get}
	\opt - P(\Alg) &\leq  \eta (e - 2) P(\Alg) - \frac{\log\epsilon^*}{\eta}
\end{align*}
\end{proof}

In order to obtain the appropriate regret bound we just need to apply
Observation~\ref{obs:even-spread} to get a bound on the value of $\epsilon^*$.
Note that by choosing $\delta = 1/T$ and assuming that $k, \sigma < T$, we get
an expected regret bound of the form $O(\sqrt{\log(T) T})$, where the
expectation is taken with respect to the random choices of both the algorithm
and nature.

\begin{corollary}
	The expected regret of Algorithm~\ref{alg:expwt} is bounded by $2\sqrt{(e-2) \log(k^2 T^3 \sigma) T} + 1$. The expectation is with respect to the random choices of the algorithm as well as nature.
\end{corollary}
\begin{proof}
	Using Observation~\ref{obs:even-spread}, we know that with except with
	probability $\delta = 1/T$, $\epsilon^* \geq 1/(k^2T^3\sigma)$; also
	$P(\Alg) \leq T$ (for any choices of nature). Substituting these two bounds
	in the statement of Theorem~\ref{thm:noregret} and setting $\eta =
	\sqrt{\tfrac{\log(k^2T^3 \sigma)}{(e-2)T}}$ gives the required bound. The $+
	1$ term accounts for the fact that with probability $\delta = 1/T$, the
	regret could be as high as $T$. 
\end{proof}

\noindent {\bf Running Time}. A na\"{i}ve implementation of
Algorithm~\ref{alg:expwt} would result in running time that is polynomial in
$t$ at time $t$ (apart from also being polynomial in $k$ and $\sigma$). This is
because the number of intervals in the refinement increase linearly in $T$.
However, by using an \emph{augmented} interval tree data structure to store the
past functions and using \emph{messages} to perform updates lazily, we can
guarantee that the running time of the algorithm is polynomial in $\log(t)$.
The sampling required in Step \ref{alg:step:select} of the algorithm can also
be implemented efficiently by storing auxiliary information about the
\emph{weights} of intervals. The next subsection explains this data structure
in detail. 

\subsection{An Efficient Data-Structure}

\label{sec:interval-trees}

\begin{figure}
  \centering
  \includegraphics[scale=0.45]{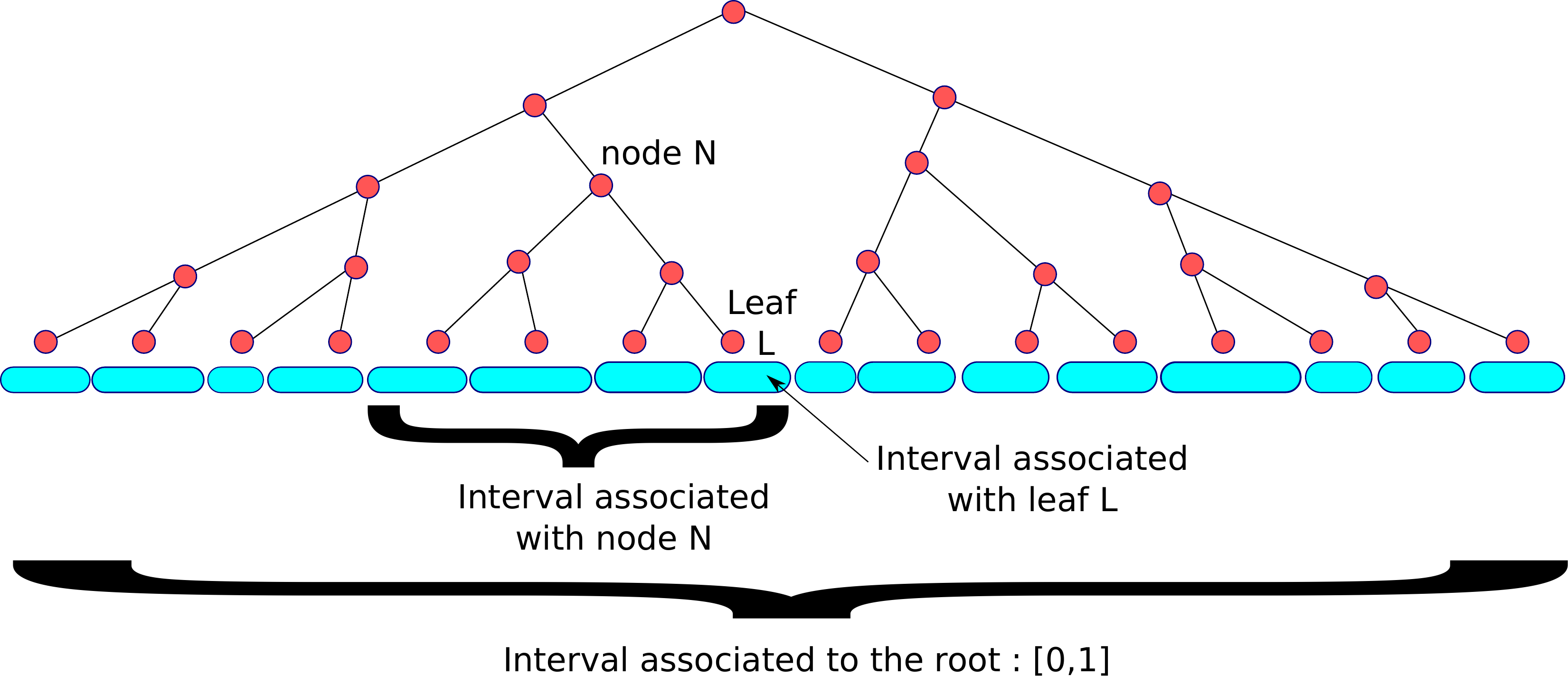}
  \caption{Example of an interval tree over the interval $[0,1]$.}
  \label{fig:tree}  
\end{figure}

\begin{figure}
  \centering
  \includegraphics[scale=0.45]{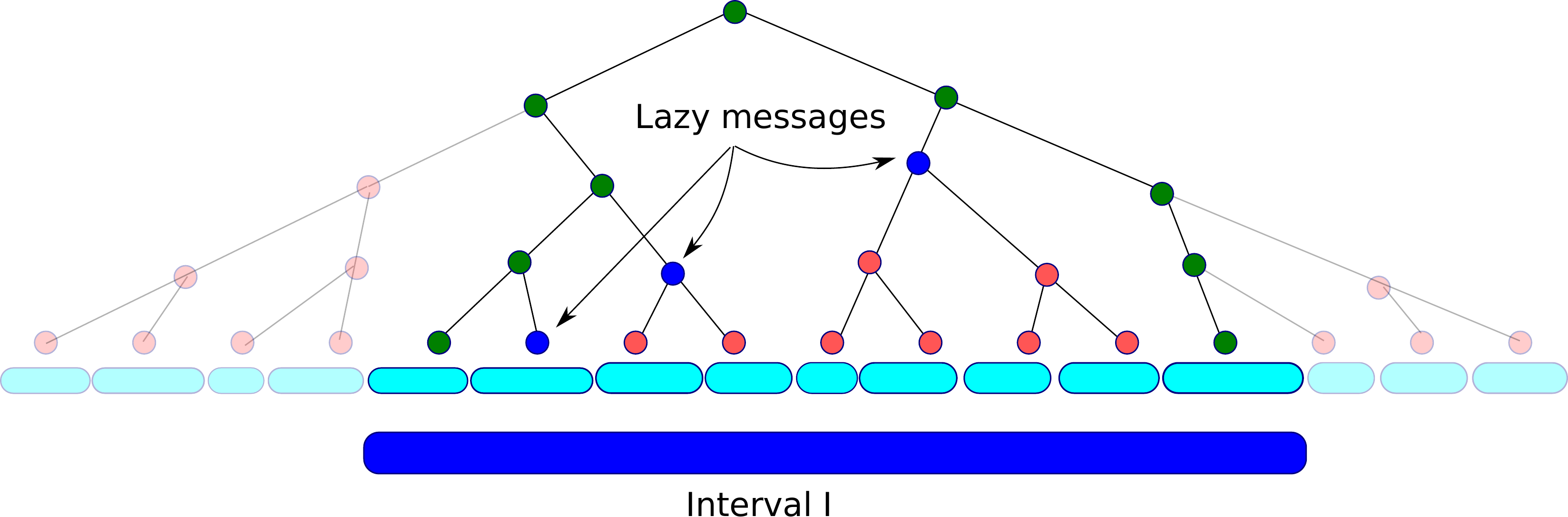}  
  \caption{Example of a call to the update procedure. The interval to update is
	  the interval I. Messages are left along the path connecting the two
	  extremities (smallest and highest intervals) of $I$. In green are the
	  nodes that get their messages updated ($m(N) = 1$ for those nodes). In
	  blue are the nodes that get a new message corresponding to $\exp(\eta
	  f_t(I))$ ($m(N) \gets m(N) \cdot \exp(\eta f_t(I))$ for those nodes). This
	  update should be propagated to all the descendants of the green nodes (the
	  red nodes), but will only be done lazily when required. The lazy approach
	  and the structure of the tree allows to bound the number of green and blue
	  vertices by $O(\log T)$. Indeed, observe that not all the nodes of the
	  subtrees corresponding to interval $I$ are updated at this step.}
  \label{fig:update}
\end{figure}

\begin{figure}
  \centering
  \includegraphics[scale=0.45]{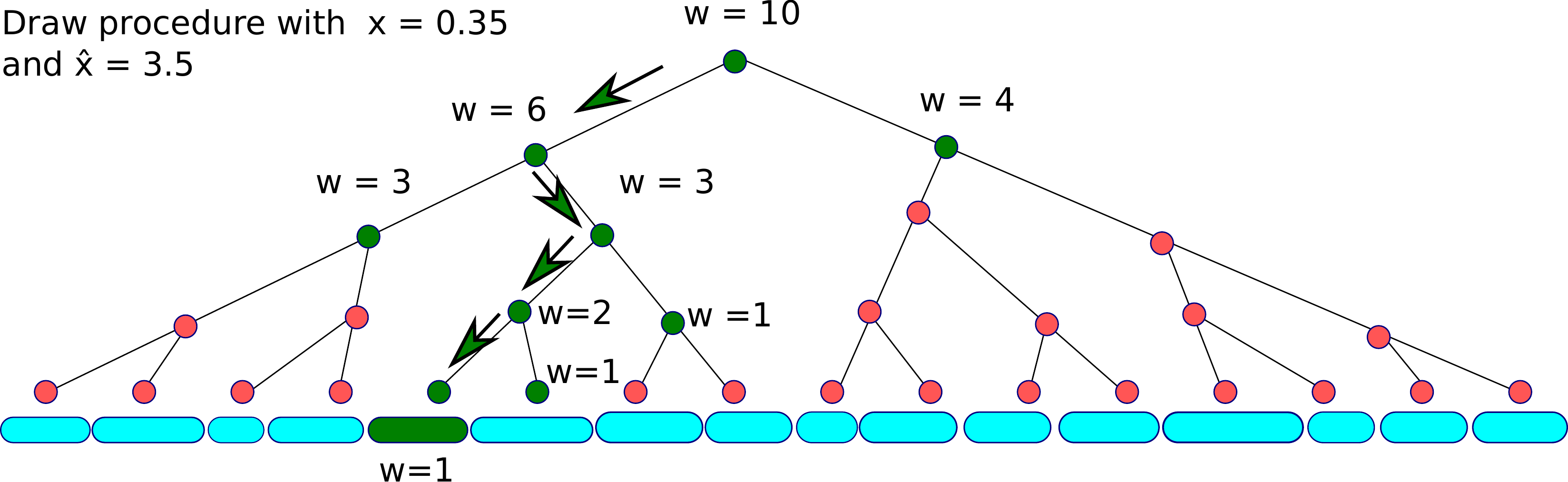}
  \caption{Example of call to the draw procedure. Assuming the value of $x$ is 0.35, the procedure moves along the
    tree based on the values of the $w(N)$ to find the interval corresponding to $p_t(x)$. The vertices in green are the vertices 
    whose pending messages are updated by the procedure (i.e.: the value of $m(N)$ for each node $N$ is $1$ after the call).}
  \label{fig:draw}  
\end{figure}

In this section, we describe a data structure that ensures that the
\emph{selection} and  \emph{update} steps (steps \ref{alg:step:select},
\ref{alg:step:update}) of Algorithm \ref{alg:expwt} have a running time of
$O(k\log (t k))$ at time $t$. 

We first describe the high level idea. In order to have an efficient
implementation, we exploit the fact that the function is piecewise constant
with $k$ pieces. We build upon a data structure called \emph{interval trees}
(see \citep{CLRS} for more details), that for any partition $P$ of the interval
$[0,1)$ into $n$ parts, maintains a binary tree with $n$ leaves that has the
following properties: 
\begin{enumerate}
	\item[(i)] The leaves are the parts of $P$. 
	\item[(ii)] Any level of the tree corresponds to a coarsening of the
		partition. More precisely, for any internal node $N$, there is an
		interval associated with $N$ which corresponds to the interval that is
		the union of the intervals associated with the children of $N$
		(see Figure \ref{fig:tree}). 
\end{enumerate}
With such a data-structure, it is possible to refine the partition (split
existing intervals), and establish membership in time $O(\log n)$, where $n$ is
the number of leaves in the tree. In our setting, $n$ will be $O(tk)$ after $t$
steps.

Here, we extend this data structure to support additional operations. Note that
$F_t$ is piecewise constant, in particular $F_t$ is constant over $I(\ell)$,
the interval defined at leaf $\ell$; by abuse of notation let $F_t(I(\ell))$
denote this constant value. For each leaf $\ell$ of the tree we will maintain a
variable $w(\ell)$ whose value is $|I(\ell)| \exp(\eta F_t(I(\ell)))$, where
$I(\ell)$ is the length of the interval corresponding to leaf $\ell$. Then, for
each internal node $N$ of the tree, whose associated interval is $I = [a,b)$,
we will maintain
	\[ w(N) = \sum\limits_{\ell \text{ leaf of the subtree}\atop \text{rooted at }
N} |I(\ell)| \exp(\eta F_t(I(\ell))) = \int_{a}^b \exp(\eta F_t(x)) dx. \]
This allows us to encode the cumulative distribution function of $p_t$ defined
at Step \ref{alg:step:distrib} of Algorithm \ref{alg:expwt}. Thus, starting
from the root and moving toward the leaves, it is possible to draw from this
distribution in time $O(\log(kt))$ (see Fig \ref{fig:draw} and the description
of the Draw procedure).

Then, at Step \ref{alg:step:update} of Algorithm \ref{alg:expwt}, we know that
the function $f_t$ is piecewise constant with $k$ pieces. For each such
\emph{piece} $I$, we need to set $F_{t + 1}(I^\prime) = F_t(I^\prime) + f_t(I)$
for each interval $I^\prime \subseteq I$ on which $F_t$ is constant. We proceed
in two steps. First, we ensure that $I$ is the disjoint union of intervals
corresponding to subtrees by splitting at most two existing intervals, those
that contain the endpoints of $I$. This operation can be implemented in
interval trees of size $n$ in time $O(\log(n))$. Then, we need to update $w(N)$
for each node $N$ of the subtrees.  Since $f_t$ is constant on the interval
$I$, we want to set $w(N) \gets w(N) \cdot \exp(\eta f_t(I))$ for each such
node $N$. However the number of such intervals may be $\Omega(tk)$ at time $t$
and doing so na\"{i}vely would result in a time complexity of $\Omega(tk)$.
Here we are aiming at time complexity $O(k \log (kT))$.  To achieve this, we
make the updates in a lazy fashion. For any interval $I$ on which $f_t$ is
constant, we consider the leaves $l$ and $h$ of the tree that contain the
extremities of $I$ (recall that membership can be computed in time $O(\log n)$
for a tree of size $n$).  We then update the values of the variables $w$ for all
the nodes along the path joining $l$ to $h$.  For each child of these nodes,
we leave a \emph{message}, that contains the value of the update for the
subtree rooted at this child. The message at a node will be applied to the node
and propagated to its children only when it is needed in the future (see Figure
\ref{fig:update}).  Thanks to the structure of the tree and since $f_t$ is
piecewise constant, we show that no information is lost through this process.

We now provide a more formal description of the data structure.  The data
structure consists of a tree where each node $N$ contains the following
information:
\begin{itemize}
  \item An \emph{interval} $I(N)$.
  \item Possibly two \emph{children} $l(N)$ and $h(N)$. Moreover, $I(l(N)) \cap I(h(N)) = \emptyset$ and $I(l(N)) \cup I(h(N)) = I(N)$ if they exist.
  \item A parent node (except for the root), $p(N)$.
  \item A \emph{weight} $w(N)$. Initially, $w(N) = |I(N)|$.
  \item A \emph{message} $m(N)$, initially $m(N) = 1$ ($m(N) = 1$ indicates that $N$ is up-to-date).
\end{itemize}

We define several operations for any given tree $\calT$.
\begin{itemize}
\item $\mathsf{UpdateMessage}(N)$: The procedure makes $N$ and its two
	children, $l(N)$ and $h(N)$ up-to-date and propagate the message toward the
	children of $l(N)$ and $h(N)$ :
  \begin{enumerate}
	  \item $w(N) \gets w(N) \cdot m(N)$
  \item $w(l(N)) \gets w(l(N) \cdot m(l(N)) \cdot m(N)$
  \item $w(h(N)) \gets w(h(N) \cdot m(h(N)) \cdot m(N)$
  \item For each child $N'$ of $l(N)$ : $m(N') \gets m(N') \cdot m(N) \cdot m(l(N))$
  \item For each child $N'$ of $h(N)$ : $m(N') \gets m(N') \cdot m(N) \cdot m(h(N))$
  \item $m(N) \gets 1$, $m(h(N)) \gets 1$, $m(l(N)) \gets 1$
  \end{enumerate}
\item $\mathsf{Insert}(I)$: The insertion procedure takes as input an interval
	$I = [l,h)$ and proceeds as the standard insertion procedure for interval
	trees in order to keep a balanced binary tree.  More precisely, it first
	splits the interval $I_0=[l_0,h_0)$ containing $l$ into two intervals $I_0^0
	= [l_0,l)$, $I_0^1[l,h_0)$ and sets $w(I_0^0) \gets |I_0^0| \cdot
	w(I_0)/|I_0|$ and $w(I_0^1) \gets |I_0^1| \cdot w(I_0)/|I_0|$.  It
	then proceeds similarly with the interval containing $h$.  In addition,
	before applying the standard insertion procedure, for every node $N$
	top-down from the root to each node that will be considered during the
	insertion, the procedure applies $\mathsf{UpdateMessage}$.  Finally, after
	the call to the standard insertion procedure, for each non-leaf node $N$
	considered by the procedure in a bottom-up fashion, it sets $w(N) \gets
	w(l(N))+w(h(N))$ and $I(N) \gets I(l(N))\cup I(h(N))$. 
\item $\mathsf{Update}(I = [l,h), w)$: Given an interval $I = [l,h)$, let
	$\ell_l$ be the leaf whose interval starts at $l$ and $\ell_h$ be the
	leaf whose interval ends at $h$. If those leaves do not exist, the
	procedure starts by calling $\mathsf{Insert}([l, h))$ to create them.
	Let $r_{l,h}$ be the lowest common ancestor of $\ell_l$ and
	$\ell_h$ in $\calT$.  Let $P_l$ be the path joining $\ell_l$ to
	$r_{l,h}$ and $P_h$ the path joining $\ell_h$ to $r_{l,h}$. Then
	the following are executed (see Figure \ref{fig:update}):
	\begin{enumerate}
		\item For every node on the path from the root to $r_{l,h}$ in a top-down
			order, apply the $\mathsf{UpdateMessage}$ procedure.
		\item For every node along the path $P_l$, from $r_{l, h}$ to $\ell_l$ in
			a top-down order, apply the $\mathsf{UpdateMessage}$ procedure.
		\item For every node along the path $P_h$, from $r_{l, h}$ to $\ell_h$ in
			a top-down order, apply the $\mathsf{UpdateMessage}$ procedure.
		\item At node $\ell_l$, update $w(\ell_l) \gets w(\ell_l) \cdot w$. For
			every internal node on the path $P_l$, \ie all nodes except for
			$\ell_l$ and $r_{l, h}$, the following updates are made in bottom-up
			order (going from $\ell_l$ to $r_{l, h}$).
			\begin{enumerate}
				\item At node $N$, if $l(N) \in P_l$, then set $w(N) \gets w(l(N))
					+ w \cdot w(h(N))$. Set $m(h(N)) \gets w$ (the earlier value of
					$m(h(N))$ was $1$ because of the $\mathsf{UpdateMessage}$
					procedure). We say that $h(N)$ has a \emph{new message}.  \smallskip \\
					{\bf Remark}: This is because the entire interval represented at
					$h(N)$ should be updated by multiplying the current weight by
					$w$. However, rather than applying this to every node in the
					subtree, we leave a message at $h(N)$ and only perform the
					updates when required in order to achieve the required time
					complexity.
				\item At node $N$, if $l(N) \not\in P_l$, then set $w(N) \gets
					w(l(N)) + \cdot w(h(N))$.
			\end{enumerate}
		\item Symmetrically, at node $\ell_h$, update $w(\ell_h) \gets w(\ell_h)
			\cdot w$. For every internal node on the path $P_h$, \ie all nodes
			except for $\ell_h$ and $r_{l, h}$, the following updates are made in
			bottom-up order (going from $\ell_h$ to $r_{l, h}$).
			\begin{enumerate}
				\item At node $N$, if $l(N) \in P_h$, then set $w(N) \gets w(l(N))
					+ \cdot w(h(N))$.
				\item At node $N$, if $l(N) \not\in P_l$, then set $w(N) \gets w
					\cdot w(l(N)) + \cdot w(h(N))$. Set $m(l(N)) \gets w$ (the
					earlier value of $m(l(N))$ was $1$ because of the
					$\mathsf{UpdateMessage}$ procedure). We say that $l(N)$ has a
					\emph{new message}.
			\end{enumerate}
		\item Finally, along every node $N$ on the path from $r_{l, h}$ to the
			root, set $w(N) \gets w(l(N)) + w(h(N))$. (Here $r_{l, h}$ and the
			root are also updated.)
  \end{enumerate}
\item $\mathsf{Draw}$: The Draw procedure starts by picking $x$ uniformly at random
	in $[0,1)$ and moves from the root toward the leaves in order to find the
	leaf interval corresponding to $x$ with respect to the distribution encoded
	by the tree $\calT$.  More precisely, the procedure starts from the root $r$
	and applies the $\mathsf{UpdateMessage}$ procedure at $r$, and then at
	$l(r)$ and $h(r)$. Then, define $\hat x = x \cdot w(r)$. The procedure
	moves to the subtree rooted at $l(r)$ if $\hat x < w(l(r))$ or to $h(r)$
	otherwise and then proceeds recursively. At a given node $N$, the procedure
	applies $\mathsf{UpdateMessage}$ at both $l(N)$ and $h(N)$. It then moves
	toward $l(N)$ if $\hat x < w(l(N))$ or toward $h(n)$ if $\hat x >
	w(h(N))$. The procedure stops when it reaches a leaf $\ell$ and returns a
	point of $I(\ell)$ uniformly at random. See Figure \ref{fig:draw}.
\end{itemize}

We use the data structure in Algorithm \ref{alg:expwt} to represent the
distributions $p_t$ (we don't need to explicitly maintain the functions $F_t$,
but if we did we could use a similar data structure). More precisely, the
algorithm starts with a tree $\calT$ containing a single node $N_0$ such that
$I(N_0) = [0,1)$ and $w(N_0) = 1$. Then at time $t$, the following are
performed
\begin{enumerate}
	\item In Step \ref{alg:step:select}, $x_t$ is drawn according to the
		$\mathsf{Draw}$ procedure described above.
	\item In Step \ref{alg:step:update}, the algorithm updates the tree
		after receiving $f_t$ as follows: for interval $I_{t, j} = [a_{t, j-1},
			a_{t, j}) $ on which $f_t$ is constant and takes value $v_{t,j}$, the
		procedure $\mathsf{Update}(I_{t, j}, \exp(\eta v_{t, j}))$ is called.
\end{enumerate}

The following claim results from the classical analysis of the complexity of interval trees (see \citep{CLRS} 
for the complete proof).
\begin{claim}
  \label{claim:complexity}
  The running time of $\mathsf{UpdateMessage}$ is $O(1)$.  Thus, by the
  definition of interval trees, the amortized running time of the
  $\mathsf{Insert}$, $\mathsf{Draw}$ and $\mathsf{Update}$ procedures is
  $O(\log (n))$ for any tree of size $n$.
\end{claim}

Let $w_t(N)$ and $m_t(N)$ denote the values of $w(N)$ and $m(N)$ at the
beginning of the $t^{\mbox{\scriptsize th}}$ iteration.  For any node $N$, define 
\[ \omega_t(N) = w_t(N) \cdot m_t(N) \cdot \prod\limits_{N' \text{ ancestor of }N} m_t(N'). \]


\begin{lemma}
  \label{lem:invariant}
  At the beginning of time step $t$ and for any node $N$,
  \begin{equation}
    \label{eq:invariant}
    \omega_t(N) = \sum_{\ell \text{ leaf descendant of }N } \omega_t(\ell).
  \end{equation}
\end{lemma}
\begin{proof}
  We show this by induction on the level $i$ of $N$ in the tree. The Lemma
  holds for any leaf $\ell$ at level $i=0$.
  Suppose it holds up to level $i-1$ and consider a node at level $i$.
  We show, by induction on $t$, that for any node $N$ of level $i$, $\omega_{t}(N) = \omega_{t}(l(N)) + \omega_{t}(h(N))$.
  If this is true, then by induction hypothesis Equation \ref{eq:invariant} holds.
  This is true for $t=1$ as the only node is a leaf.
  We assume that this holds up to $t-1$ and show that it holds for time $t$. Note that any call to $\mathsf{UpdateMessage}$ is made
  through a call to the $\mathsf{Draw}$, $\mathsf{Insert}$, or $\mathsf{Update}$ procedures, so we only consider that a call to one of those three procedures occurred at time $t$.
  Observe that in those three procedures, the $\mathsf{UpdateMessage}$ procedure is applied top-down from the root toward one (in the case
  of the Draw procedure) or two (in the case of the $\mathsf{Insert}$ and $\mathsf{Update}$ procedures) leaves.
  We consider a node $N$ at distance at least one of all the paths followed by $\mathsf{UpdateMessage}$ and such that none of its ancestors received
  a new message at time $t$. We show that the lemma holds for $N$.
  First if $N$ is at distance at least 3 from the paths, we have
  that $\prod_{N' \text{ ancestor of } N} m_{t-1}(N') = \prod_{N' \text{ ancestor of } N} m_t(N')$ and $\omega_{t-1}(N) = \omega_{t}(N)$
  since the messages are applied top-down and by definition of the $\mathsf{UpdateMessage}$ procedure and none of its ancestors received a new message.
  This holds for any descendant $N'$ of $N$ and so $\omega_{t}(N) = \omega_{t}(l(N)) + \omega_{t}(h(N))$ by induction hypothesis.

  Now suppose that $N$ is at distance 2 from the paths. We have that
  $$m_{t-1}(N) \cdot \prod_{N' \text{ ancestor of } N} m_{t-1}(N') = m_t(p(N)) \cdot m_t(N).$$
  Additionally, $w_t(N) = w_{t-1}(N)$. Thus, $\omega_{t}(N) = \omega_{t-1}(N)$. Since, by the above discussion,
  this also holds for $l(N)$ and $h(N)$. We have by induction hypothesis $\omega_{t}(N) = \omega_{t}(l(N)) + \omega_{t}(h(N))$.

  We now consider a node $N$ at distance exactly 1 from the paths. 
  Since the $\mathsf{UpdateMessage}$ are applied top-down we have and none of its ancestor received a new message,
  $w_t(N) \cdot m_t(N) = w_{t-1}(N) \cdot m_{t-1}(N) \cdot  \prod_{N' \text{ ancestor of } N} m_{t-1}(N') = \omega_{t-1}(N)$.
  We observe that $\prod_{N' \text{ ancestor of } N} m_{t}(N') = 1$. Thus, $\omega_t(N) = \omega_{t-1}(N)$.
  Since $l(N)$ and $h(N)$ are at distance 2, $\omega_{t}(N) = \omega_{t}(l(N)) + \omega_{t}(h(N))$ by induction hypothesis.

  We now turn to prove that $\omega_{t}(N) = \omega_{t}(l(N)) + \omega_{t}(h(N))$ for the remaining nodes, \ie for the
  nodes in the paths and for the nodes for which an ancestor received a new message.
  Suppose first the Draw procedure was called at time $t$. We consider the path of nodes $P$ taken by the procedure.
  For any node in the path we have 
  $w_t(N) = w_{t-1}(N) \cdot m_{t-1}(N) \cdot  \prod_{N' \text{ ancestor of } N} m_{t-1}(N') = \omega_{t-1}(N)$.
  Remark that $\prod_{N' \text{ ancestor of } N} m_{t}(N') = m_t(N)= 1$. Thus, $\omega_t(N) = \omega_{t-1}(N)$.
  Now, by the above discussion, the $\omega_t(l(N)) = \omega_{t-1}(l(N))$ and $\omega_t(h(N)) = \omega_{t-1}(h(N))$.
  Thus, $\omega_{t}(N) = \omega_{t}(l(N)) + \omega_{t}(h(N))$
  
  Since the procedure also applies $\mathsf{UpdateMessage}$ to $l(N)$ and $h(N)$, we have $m(l(N)) = m(h(N)) = 1$.
  Thus, $\omega_t(l(N)) = w(l(N))$ and $\omega_t(h(N)) = w_t(h(N))$.
  Observe that the procedure sets $w_t(N) \gets w_t(h(N)) + w_t(l(N))$ and therefore, $\omega_{t}(N) = \omega_{t}(l(N)) + \omega_{t}(h(N))$.

  We now assume that a call to the $\mathsf{Insert}$ procedure occurred at time $t$.
  Since $N$ is in the path we have that $\prod_{N' \text{ ancestor of } N} m_{t}(N') = m_t(N)= 1$ and 
  the procedure sets $w_t(N) \gets w_t(l(N)) + w_t(h(N))$.
  Since the procedure also applies $\mathsf{UpdateMessage}$ to $N$, we have $m(l(N)) = m(h(N)) = 1$.
  Thus, $\omega_t(l(N)) = w(l(N))$ and $\omega_t(h(N)) = w_t(h(N))$.
  Observe that the procedure sets $w_t(N) \gets w_t(h(N)) + w_t(l(N))$ and therefore, $\omega_{t}(N) = \omega_{t}(l(N)) + \omega_{t}(h(N))$.
    
  Therefore, we turn to the case where an update on an interval $I_0$ occurred at time $t$, \ie a call to function
  $\mathsf{Update}$ for an interval $I=[l,w]$ and a value $w$.
  Observe that for any node $N$ such that $I(N) \cap I = \emptyset$, $N$ is not on the paths and none of the ancestors of $N$ received
  a new message.
  Thus, consider a node $N$ such that $I(N) \cap I \neq \emptyset$.  
  Suppose $N$ is in $P_{l}$ or in $P_{h}$ or in the path between the root and $r_{l,h}$.
  Since $\mathsf{UpdateMessage}$ in applied top-down to all the ancestors of $N$ and to $N$ we have 
  $\prod_{N' \text{ ancestor of }N} m(N') = m(N) = m(l(N)) = m(h(N)) = 1$, and so,
  $\omega_t(N) = w_t(N)$.
  Moreover, by definition of the procedure $w_t(N) = w_t(l(N)) + w_t(h(N)) = \omega_t(l(N)) + \omega_t(h(N))$
  and thus, Equation \ref{eq:invariant} holds by induction hypothesis on the level of $N$.

  Now suppose $N$ is not in $P_{l}$ or in $P_{h}$ or in the path between the root and $r_{l,h}$.
  Consider the lowest ancestor $N'$ of $N$ in $P_l$ or $P_h$ and suppose, w.l.o.g, that $N' \in P_l$.
  Then since $N \notin P_l$ we have that $l(N') \in P_l$ and $h(N')$ is either $N$ or an ancestor of $N$.
  In both cases we have that $\omega_t(N) = \omega_{t-1}(N) \cdot w$. Now, if $N$ has a descendant $N''$, we have
  that $\omega_t(N'') = \omega_{t-1}(N'')  \cdot w$ as well. Thus $\omega_t(N) = \omega_t(l(N)) + \omega_t(h(N))$.
\end{proof}

\begin{lemma}
  \label{lem:treeCDF}
  The Draw procedure produces a point $p$ according to the distribution described at Step \ref{alg:step:distrib} of Algorithm \ref{alg:expwt}. 
\end{lemma}
\begin{proof}
  We start by showing the following invariant:

  \emph{At any timestep $t$, for any leaf $\ell$, for any $p \in I(\ell)$,}
  \begin{equation}
    \label{eq:leaves}
    \omega_t(\ell) = |I(\ell)| \exp(\eta F_t(p)).
  \end{equation}

  We proceed by induction on $t$. Note that for $t =1$, it holds by definition. Suppose now that this holds up to time $t-1$ and
  consider time $t$. At time $t-1$ a call to the $\mathsf{Update}$ procedure was made, as otherwise by induction hypothesis, since
  no new message was inserted the invariant holds.
  Let $I=[l,h]$ be the interval that was updated and $\exp(\eta f_t(p))$ be the value of the update for some $p \in I$.
  Observe that, by the definition of the $\mathsf{Update}$ procedure, if a leaf $\ell$ is such that  $I(\ell) \cap I = \emptyset$,
  then $\omega_t(N) = \omega_{t-1}(N)$.

  Thus, consider a leaf $\ell$ such that $I(\ell) \cap I \neq \emptyset$.
  Suppose first $l \in I(\ell)$ or $h \in I(\ell)$. Then the procedure sets
  $w(\ell) \gets w(\ell) \cdot \exp(\eta f_t(p))$ and since $\prod_{N' \text{ ancestor of }N} m(N') = m(N) = m(l(N)) = m(h(N)) = 1$,
  $w_t(\ell) = \omega_t(\ell)$.
  By Induction hypothesis, this means that
  $\omega_t(\ell) = |I(\ell)| \exp(\eta F_{t-1}(p)) \cdot \exp(\eta f_t(p)) = |I(\ell)| \exp(\eta (F_{t-1}(p) + f_t(p)))$.
  Since $F_t = F_{t-1} + f_t$, $\omega_t(\ell) = |I(\ell)| \exp(\eta (F_{t}(p))$.
  
  Now consider a leaf $\ell$ such that $l \notin I(\ell)$ and $h \notin I(\ell)$ and $I(\ell) \cap I \neq \emptyset$.
  It follows that $\ell$ has an ancestor in either $P_l$ or $P_h$.
  Consider the lowest ancestor $N$ of $\ell$ in $P_l$ or $P_h$ and suppose, w.l.o.g, that $N \in P_l$.
  Observe $\prod_{N' \text{ ancestor of }N} m(N') = m(N) = m(l(N)) = m(h(N)) = 1$.
  Then since $\ell \notin P_l$ we have that $l(N) \in P_l$ and $h(N)$ is either $\ell$ or an ancestor of $\ell$.
  In both cases we have that $\omega_t(\ell) = \omega_{t-1}(\ell) \cdot \exp(\eta f_t(p))$.  
  Again, since $F_t = F_{t-1} + f_t$, $\omega_t(\ell) = |I(\ell)| \exp(\eta (F_{t}(p))$.
  We conclude that the invariant holds.

  Now, observe that $F_t(x)$ is constant on any interval $I$ such that there exists a leaf $\ell$ with $I(\ell) = I$. Denote by $p_{\ell}$ an
  arbitrary point of $I(\ell)$.
  Hence, combining Equations \ref{eq:invariant} and \ref{eq:leaves}, for any node $N$,
  $$\omega_t(N) = \sum\limits_{\ell \text{ leaf of the subtree}\atop \text{rooted at } N} |I(\ell)| \exp(\eta F_t(p_{\ell})) = \int_{a}^b \exp(\eta F_t(x)) dx,$$
  where $I(N) = [a,b]$.
  Moreover, remark that when the Draw procedure reaches a node $N$, for each ancestor $N'$, we have $m(N') = 1 = m(N)$.
  It follows that we have $\omega_t(N) = w_t(N)$.
  Therefore, by Lemma \ref{lem:invariant},
  the procedure draws a point $p$ according to the distribution described at Step \ref{alg:step:distrib} of Algorithm \ref{alg:expwt}.
\end{proof}

From the previous Lemma and Claim \ref{claim:complexity}, we deduce the following Theorem.

\begin{theorem}
  \label{thm:complexity}
  The \emph{decision step} (step \ref{alg:step:select}) of Algorithm \ref{alg:expwt} can be performed in time $O(\log T k)$.
  Moreover the \emph{update step} (step \ref{alg:step:update}) can be performed in time $O(k  \log (T k))$.
\end{theorem}

\section{Bandit Setting}

In the bandit setting, we only observe the value $f_t(x_t)$. Thus, we need to
estimate $f_t$ elsewhere. This is made difficult by the fact that $f_t$ may
have discontinuities. We construct an estimator $\hat{f}_t$ which takes an
appropriately re-scaled value on a small interval around $x_t$ and is $0$
elsewhere. If this chosen interval is too large, then it is quite likely that a
point of discontinuity of $f_t$ lies in this interval. If it is small and
no point of discontinuity of $f_t$ lies in the interval, then $\hat{f}_t$ is an
unbiased estimator of $f_t$; however, choosing too small an interval may make
$\hat{f}_t(x_t)$ very large. Tuning the length of the interval and a careful
analysis gives us a regret bound of $\tilde{O}(T^{2/3})$. 


\begin{algorithm}
	\alginput: $\eta,~\mu,~\gamma$ all positive and satisfying $1/\mu \in \naturals$, $\gamma \leq \tfrac{1}{2}$, $\eta \leq \gamma \mu$ \smallskip \\
  \algset~$\calI = \langle [(i-1)\mu, i \mu) \rangle_{i = 1}^{1/\mu}$ be a family of intervals. \smallskip \\
  \algset~$w_1(x) = 1$ to be the constant $1$ function over $[0, 1)$ \smallskip \\
  \algfor~ $t = 1, 2, \ldots, T$ \algdo
  \begin{enumerate}
  \item Define $p_t(x) = (1-\gamma) \frac{w_t(x)}{\int_0^1 w_t(x) dx} + \gamma$
  \item Pick $x_t \sim p_t$ \label{alg:B:step:select}
  \item Let $I_t$ be the interval of $\calI$ that contains $x_t$.
  \item Observe function $f_t(x_t)$. Receive payoff $f_t(x_t)$.
  \item Set $\hat f_t(x) = \frac{f_t(x)}{p_t(I_t)}$ for all $x \in I_t$ and $\hat f_t(x) = 0$ for all $x \in [0,1) \setminus I_t$, \\
    Here for interval $I$, $p_t(I) = \Pr_{x \sim p_t}( x \in I)$.
  \item Set $w_{t+1}(x) = w_{t}(x) \cdot \exp(\eta \hat f_t(x))$ for all $x \in [0,1)$.
  \end{enumerate}
  \caption{Bandit Algorithm \label{alg:Bwt}}
\end{algorithm}

\begin{theorem}
The expected regret of Algorithm~\ref{alg:Bwt} is bounded by 
\[ 2 \gamma T + 2 \frac{\eta}{\mu} T + \frac{1}{\eta}
	\ln\left({\mu}^{-1} \right) + k \sigma \mu T \] 
The expectation is taken with respect to the random choices made by the
algorithm as well as those by the adversary/nature. For suitable
choices of $\mu$, $\gamma$, and, $\eta$, this gives a regret bound of
$O(\poly(k, \sigma, \log(T)) T^{2/3})$.
\end{theorem}
\begin{proof} 
We follow the analysis as in the expert and semi-bandit cases. Let $W_t = \int_0^1 w_t$. We show that:
 \begin{align}
    \ln\left(\frac{W_{t + 1}}{W_{t}}\right) &\le \frac{\eta}{1-\gamma} \int_0^1 p_t(x) \hat f_t(x) dx + \frac{(e-2)\eta^2}{1-\gamma}\int_0^1 p_t(x) \hat f_t(x)^2 dx 
    \label{eq:S:UB_w}
  \end{align}
  
  By Definition of $W_t$,
  \begin{align}
    \frac{W_{t+1}}{W_{t}} &= \frac{\int_0^1 w_t(x) \exp(\eta \hat f_{t}(x)) dx}{\int_0^1 w_t(x) dx} \nonumber \\
	 &= \frac{1}{1-\gamma} \int_0^1 (p_t(x)- \gamma) \exp(\eta \hat f_{t}(x)) dx
	 \label{alignarray-step1} \\
	 &\le \frac{1}{1-\gamma} \int_0^1 (p_t(x)- \gamma)( 1+ \eta \hat f_t (x) + (e-2)\eta^2 \hat f_t(x)^2) dx \label{alignarray-step2} \\
	 &\le 1 + \frac{\eta}{1-\gamma} \int_0^1 p_t(x) \hat f_t(x) dx + \frac{(e-2)\eta^2}{1-\gamma} \int_0^1 p_t(x) \hat f_t(x)^2 dx \label{alignarray-step3} 
  \end{align}
  Above (\ref{alignarray-step1}) holds since $p_t(x) = (1-\gamma)
	\frac{W_t(x)}{\int_0^1 W_t(x) dx} + \gamma$, (\ref{alignarray-step2}) holds
	as $\exp(z) \leq 1 + z + (e - 2) z^2$ for $z \in [0, 1]$---the requirement
	that $\eta \leq \gamma \mu$ implies that this is always satisfied, since
	$p(I_t) \geq \gamma \mu$ and $f_t(x_t) \leq 1$. Taking logarithms of both
	sides in (\ref{alignarray-step3}) and using the fact that $\ln(1 + z) \leq
	z$ we get (\ref{eq:S:UB_w}).

  Thus, by summing the inequality (\ref{eq:S:UB_w}) for $t = 1, \ldots, T$, we get
\begin{align}
  \label{eq:S:UB_w2}
\ln\left(\frac{W_{T + 1}}{W_1}\right) \leq \frac{\eta}{1-\gamma} \sum\limits_{t=1}^T \int_0^1 p_t(x) \hat f_t(x) dx + \frac{(e-2)\eta^2}{1-\gamma}
\sum\limits_{t=1}^T \int_0^1 p_t(x) \hat f_t(x)^2 dx
\end{align}

First, we observe that $\int_0^1 p_t(x) \hat{f}_t(x) = f_t(x_t)$. Also, we use the fact that 
	\[ \int_0^1 p_t(x) (\hat{f}_t(x))^2 dx = f_t(x_t) \hat{f}_t(x_t) \leq \hat{f}_t(x_t) \]

Thus, we can rewrite (\ref{eq:S:UB_w2}) to give us:
\begin{align}
  \label{eq:S:UB_w3}
\ln\left(\frac{W_{T + 1}}{W_1}\right) \leq \frac{\eta}{1-\gamma} \sum\limits_{t=1}^T f_t(x_t)  + \frac{(e-2)\eta^2}{1-\gamma}
\sum\limits_{t=1}^T  \hat f_t(x_t) 
\end{align}

Note that since $w_1 \equiv 1$, $W_1 = 1$. 
Let $x^*$ be some point in $[0, 1)$ that is a maximizer of $\sum_{t = 1}^T f_t(x)$. 
With probability $1$, there exists $x^*$ that lies in the interior of some interval defined by $\calI$, 
	denote this interval by $I(x^*)$. Since for all $t$, $\hat{f}_t$ is constant over $I(x^*)$, we have:
	\begin{align}
		\ln\left( \frac{W_{T + 1}}{W_1} \right) &\geq \eta \sum_{t = 1}^T \hat{f}_t(x^*) + \ln(|I(x^*)|) \label{eq:S:LB_w}
	\end{align}

	Putting (\ref{eq:S:UB_w3}) and (\ref{eq:S:LB_w}) together, rearranging some terms and dividing by $\eta$, we get:
 	\begin{align}
		\sum_{t = 1}^T \hat{f}_t(x^*) - \sum_{t = 1}^T f_t(x_t) &\leq \frac{\gamma}{1 - \gamma} \sum_{t = 1}^T f_t(x_t) + \frac{(e-2)\eta}{1 - \gamma} \sum_{t = 1}^T \hat{f}_t(x_t) - \frac{\ln(|I(x^*)|)}{\eta} 
\intertext{We use the fact that $\gamma \leq 1/2$, thus $1/(1 - \gamma) \leq 2$
to simplify some terms.  By adding $\sum_{t = 1}^T f_t(x^*)$ on both sides and
rearranging terms, we get,}
\label{eqn:S:regret-bound}
\sum_{t = 1}^T {f}_t(x^*) - \sum_{t = 1}^T f_t(x_t) &\leq 2{\gamma} \sum_{t =
1}^T f_t(x_t) + 2{(e-2)\eta} \sum_{t = 1}^T \hat{f}_t(x_t) -
\frac{\ln(|I(x^*)|)}{\eta} + \sum_{t = 1}^T {f}_t(x^*) - \sum_{t = 1}^T
\hat{f}_t(x^*) 
\end{align}

The LHS of (\ref{eqn:S:regret-bound}) is exactly the regret suffered by the
algorithm. Thus, to bound the expected regret, we only need to bound the
expectation of the RHS, where the expectation is taken with respect to both the
random choices of the algorithm and nature.

For the first term on the RHS of (\ref{eqn:S:regret-bound}), we simply use the
fact that $f_t(x_t) \leq 1$, hence the term is bounded by $2 \gamma T$. For the
second term, we first only look at the expectation with respect to the choice
of $x_t \sim p_t$ to get, 
\[ \E_{x_t \sim p_t}[\hat{f}_t(x_t)] = \sum_{I \in \cal I} p_t(x_t \in I) \cdot \frac{f_t(x_t)}{p_t(x_t \in I)} \leq |\calI| = \frac{1}{\mu} \]
and 
hence the expectation of the second term is bounded by $2 (e - 2) \eta T /\mu$.
The intervals in $\calI$ all have length $\mu$. Thus, $-\ln(|I(x^*)|)/\eta =
\ln(\mu^{-1})/\eta$. 

Finally, we deal with the last term in the RHS of (\ref{eqn:S:regret-bound}).
However, there is a subtle issue which requires us to use the fact that we are
playing against an \emph{oblivious} adversary. The point $x^*$ depends on the
random choices made by nature/adversary (denoted by $\F_{adv}$). Let $\calH_{t
- 1}$ denote the history of random choices made by the algorithm up to time $t
- 1$.  Thus, we can compute $\E[ \hat{f}_t(x^*) ~|~ \F_{adv}, \calH_{t-1}]$,
where the expectation is only with respect to the random choices made by the
algorithm at time $t$.  Let $Z_t$ be $1$ if the interval $I_t$ is not a
sub-interval of some interval on which $f_t$ is constant, \ie $I_t$ is not a
sub-interval of one of the pieces used to define $f_t$. First, we note that 
\[ f_t(x^*)	- \E[\hat{f}_t(x^*) ~|~ \F_{adv}, \calH_{t-1}] \leq Z_t, \]
where the expectation is only with respect to $x_t \sim p_t$. This holds since,
if $Z_t = 0$, then $\E[\hat{f}_t(x^*) ~|~ \F_{adv}, \calH_{t - 1}] =
f_t(x^*)$, else, the LHS is at most $f_t(x^*) \leq 1 = Z_t$.

Thus, we only need to bound the $\E[Z_t]$ which only depends on the random
choices made by the nature/adversary. We consider the distributions that define
the points of discontinuities of $f_t$. Depending on the order in which the
adversary chooses to exercise their random choices, they may not know exactly
the interval $I(x^*) \in \calI$. However, the probability of inserting points
of discontinuity in $I(x^*)$ can only increase if we assume that they know
$I(x^*)$. To maximize the probability of this event they can put as much of
probability mass in $I(x^*)$ for each of the $k$ points of discontinuity as
possible. Since the density function of the distributions defining the points
of discontinuity is bounded by $\sigma$, this probability cannot exceed $k \mu
\sigma$ (since $|I(x^*)| = \mu$). Thus, $\E[Z_t] \leq k \mu \sigma$. Thus, the
expectation of the last term in the RHS of (\ref{eqn:S:regret-bound}) can be
bounded by $k\sigma \mu T$. 

Putting everything together, we get,
\[ \E[\Regret(\Alg)] \leq 2 \gamma T + 2 \frac{\eta}{\gamma \mu} T +  \frac{1}{\eta} \ln\left(\mu^{-1}\right)  +  
  k \sigma \mu T \]
Optimizing the RHS of the above equation with respect to $\mu$, $\gamma$ and
$\eta$ (while maintaining $\gamma \leq 1/2$ and $\eta \leq \gamma \mu$), gives
that the expected regret is bounded by $O(\poly(k, \sigma, \log(T)) T^{2/3})$.
\end{proof}

\noindent{\bf Running Time}. As in the full-information setting, the running
time of our algorithm can be made polynomial in $\log(t)$ by using the
\emph{augmented} interval trees defined in Section~\ref{sec:interval-trees}.

\section{Applications and Experiments}

\label{sec:experiments}
This section illustrates our framework with concrete examples. We describe
three problems for which greedy approaches (also called \emph{priority
algorithms} by \cite{BNR:03}) are often used in practice. We provide
experimental results that highlight the efficiency of our approach.

We recall the definitions introduced by \cite{GR2016}. An optimization problem
$\Pi$ is an \emph{object assignment} problem if its input consists of a set of
objects with so-called attributes and a solution to $\Pi$ is a function $S :
[n] \rightarrow \reals_+$ where $n$ is the number of objects, and $[n]$ denotes
the set $\{1, \ldots, n\}$.  For a given instance of an object assignment
problem, the \emph{attribute} of object $i$, $\xi_i$, is an ordered set of
$\ell$ real values.  Let $\Xi_1, \ldots, \Xi_{\ell}$ be the sets of possible
values for attribute $1, \ldots, \ell$, respectively.  For example, the
knapsack problem is an object assignment problem: its input is a set of items
with two attributes, value and size, in $\reals_+$.  A solution to the problem
is a function $S : [n] \rightarrow \{0,1\}$, where $S(i) = 1$ if and only if
the $i$th item is part of the solution.

Define the functions $\lambda : \reals_+ \times \Xi_1 \times \ldots \times
\Xi_{\ell} \rightarrow \reals_+$ as \emph{single-parameter scoring rules},
where the first argument is the \emph{parameter}.  A \emph{family} of
single-parameter scoring rules is a set of scoring rules of the form
$\lambda(\rho, \xi_i)$ for each parameter value $\rho$ in some interval $I
\subseteq \reals$ and such that $\lambda$ is continuous in $\rho$ for each
fixed value of $\xi_i$.  We say that a family of single-parameter scoring rules
is \emph{$\kappa$-crossing} if for each $\xi_i,\xi_j, \xi_i \neq \xi_j$ there
are at most $\kappa$ values of $\rho$ for which $\lambda(\rho,\xi_i) =
\lambda(\rho,\xi_j)$.

Define an \emph{assignment rule} $\alpha_{\lambda}$ as a function which given an object $i$ and the value $\lambda(\rho, \xi_i)$
for some single-parameter scoring rule $\lambda$ with parameter value $\rho$,
computes the value of $S(i)$ and possibly modifies the attributes $\xi \setminus \{\xi_i\}$.
We say that an assignment rule is \emph{$\beta$-bounded} if for any object $i$, the number of different values that the attribute $\xi$ can take
is at most $\beta$.

We are now ready to introduce the notion of \emph{greedy heuristic}.
An algorithm $A$ for an object assignment problem is a greedy heuristic if it greedily computes a solution $S$ according
to a scoring-rule $\lambda$ and an assignment rule $\alpha_{\lambda}$.
Thus, a family of $\kappa$-crossing scoring rules and a $\beta$-bounded assignment rule defines a $\kappa,\beta-$\emph{family of greedy heuristics}.

\cite{GR2016} show that the outputs of a $\kappa,\beta$-family of greedy heuristics on an instance of size $n$ is a
piecewise constant function with $O((sn\beta)^2 \kappa)$ pieces. Therefore we propose to apply our approach to the following problems.

\subsection{Examples of Applications}
\label{sect:examples}
We exhibit families of greedy heuristics for three classical $\mathrm{NP}$-hard optimization problems.

\emph{The Knapsack problem:} The input is a set of $n$ pairs value and size, $(v_i,s_i)$, and a capacity $C$.
The objective is to output a subset $S \subseteq [n]$ such that $\sum_{j \in S} s_j \le C$ and $\sum_{j \in S} v_j$ is maximum.

\emph{A family of greedy heuristics for Knapsack:} 
Given a parameter $\rho$ and an instance $\xi = \{\xi_1 = (v_1,s_1),\ldots,\xi_n=(v_n,s_n)\}$, the greedy heuristic performs the following computations.
It first orders the elements by non-decreasing values of $v_i/(s_i)^{\rho}$. Then the heuristic greedily (subject to feasibility)
adds objects to the solution in this order. Note that the cases $\rho = 0$ and $\rho = 1$ are classical heuristic for knapsack.

\emph{The Maximum Weighted Independent Set problem (MWIS):} The input is a graph $G = (V,E)$, and weights $w_i$ for each of the $n$ vertices.
The objective is to output a subset $S \subseteq [n]$ such that for any $i,j \in S$, $(i,j) \notin E$ and $\sum_{j \in S} w_j$ is maximum.

\emph{A family of greedy heuristics for MWIS:} Denote by $N(i)$ the set of
neighbors of $i$ in $G$.  Given a parameter $\rho$ and an instance $\xi =
\{\xi_1 = (w_1,N(1)) ,\ldots,\xi_n=(w_n,N(n))\}$, we consider the
\emph{adaptive} greedy heuristic.  It adds all the degree 0 vertices to the
solution and adds the vertex maximizing $w_i/|N(i)|^{\rho}$, then removes
vertex $i$ and all the vertices of $N(i)$ from the graph, updates $N(j)$ for
the remaining vertices $j$, and repeats until there are no more vertices in the
graph.

The Knapsack and MWIS problems were studied by \cite{GR2016}.
Here, we also consider the weighted $k$-means problem.

\emph{The weighted $k$-means problem :} The input is a set of $n$ points, a
metric $d : [n] \times [n] \rightarrow \reals_+$, and a set of weights
$\{w_1,\ldots,w_n\}$. The objective is to output a subset $S \subseteq [n]$ of size $k$ such that  $\sum_{i \in [n]} \min_{j \in S} w_i \cdot d(i,j)^2$ is
minimized.

\emph{A family of greedy heuristics for $k$-means:}
Given a parameter $\rho$ and an instance $\xi = \{\xi_1 = (\{d(1,i)\mid i \in [n]\}) ,\ldots, \xi_n=(\{d(n,i)\mid i \in [n]\}))\}$,
we consider the \emph{adaptive} greedy heuristic, which can be seen as a generalization of Gonzalez' algorithm \citep{G85}
for the $k$-center problem (when $\rho = \infty$).\\
\begin{algorithm}
	\alginput: $k$, $\rho$, $\xi = \{\xi_1 = (\{d(1,i)\mid i \in [n]\}) ,\ldots, \xi_n=(\{d(n,i)\mid i \in [n]\})\}$, \smallskip \\
	\algset~ $S \gets \emptyset$ \smallskip \\
	\algfor~ $t = 1, 2, \ldots, k$ \algdo
        \begin{enumerate}
          \item Add to $S$ the point $p$ that maximizes $ \max_{i \in S} w_p/d(p,i)^{1/\rho}$.
        \end{enumerate}
        Return $S$
	\caption{Adaptive Greedy for weighted $k$-means.}
\end{algorithm}

\subsection{Experiments}
In this section, we describe experimental results to illustrate the efficiency of Algorithm \ref{alg:expwt}.
\begin{figure}[t!]
  \centering
  \begin{minipage}[b]{0.45\textwidth}
    \centering
    \includegraphics[width=1\linewidth]{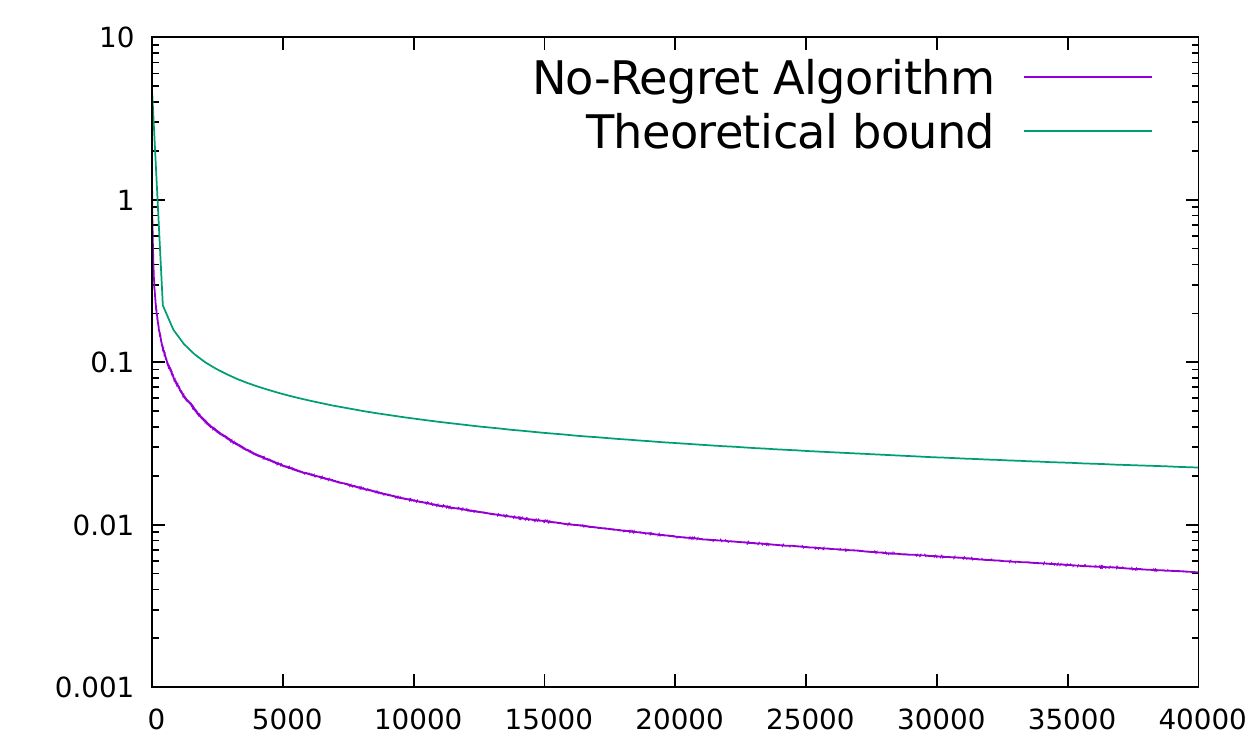}
    \subcaption{Knapsack}
    \label{fig:knap}
  \end{minipage}
  \begin{minipage}[b]{0.45\textwidth}
    \centering
    \includegraphics[width=1\linewidth]{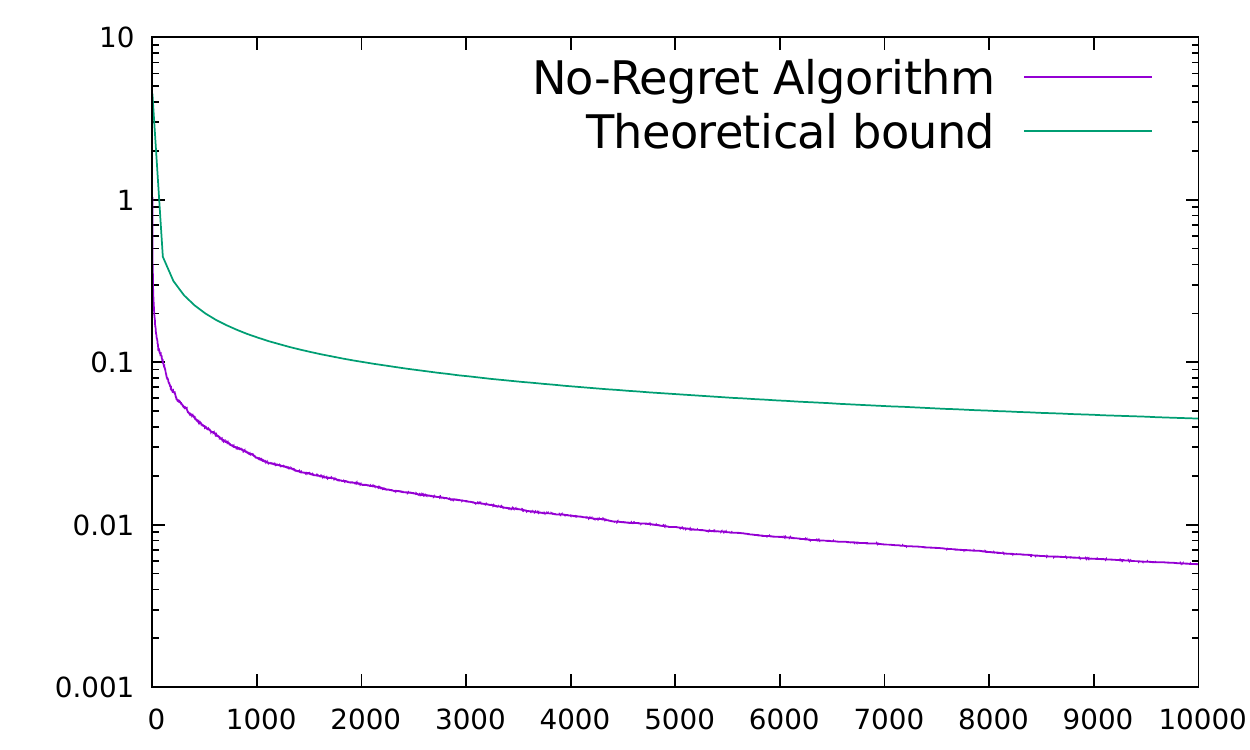}
    \subcaption{MWIS}
    \label{fig:mwis}
  \end{minipage}
  \begin{minipage}[b]{0.45\textwidth}
    \centering
    \includegraphics[width=1\linewidth]{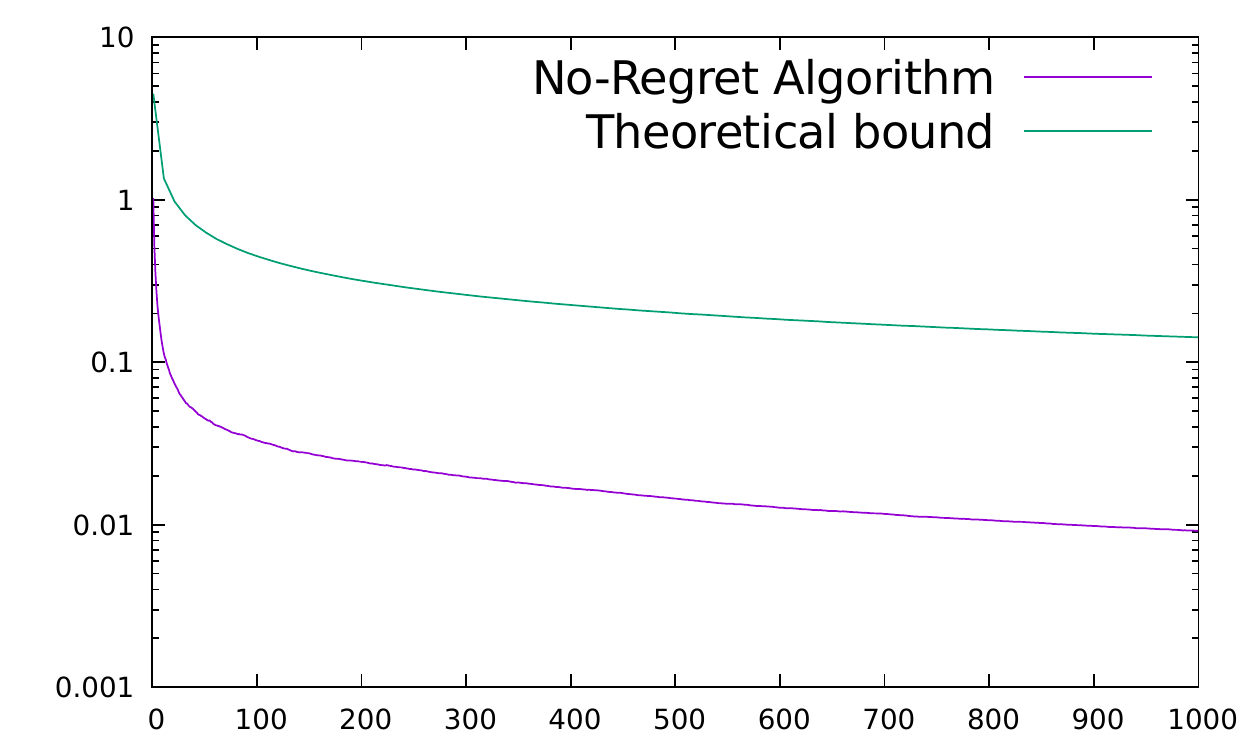}
    \subcaption{Weighted $k$-means}
    \label{fig:kmeans}
  \end{minipage}

  \caption{Average per-round regret of Algorithm \ref{alg:expwt} for Knapsack, MWIS and weighted $k$-Means, and the theoretical bound of $2\eta$.
    The $x$-axis represents the timesteps. The $y$-axis represents the per-round regret.}
\end{figure}

\emph{Knapsack:} A random instance for knapsack consists of $n = 20$ real numbers for the values picked uniformly and independently
at random in the interval $[0,1)$ together with $n$ real numbers for the sizes picked uniformly and independently at random in the interval $[0,1]$
and a capacity bound of 1.
We run Algorithm \ref{alg:expwt} for the family of greedy
heuristics described in Section \ref{sect:examples} for $\rho \in [0,1]$ and for $T = 40000$ steps. We performed $100$ rounds of the algorithm.
Figure \ref{fig:knap} shows the mean per-round regret of Algorithm \ref{alg:expwt} measured after $t$ timesteps.

\emph{MWIS:} For the MWIS problem, we consider random instances consisting of a random Erd\H{o}s-R\'{e}nyi graph with $n = 20$ vertices and
with edge probability $p \in [0.1, 0.5]$ and $n$ values for the weights picked uniformly and independently at random in the interval $[0,1]$ .
We run the family of greedy heuristics described in section \ref{sect:examples}, with $\rho \in [0,1]$ and for $T=10000$.
We performed $100$ rounds of the algorithm.
Again, Figure \ref{fig:mwis} shows the mean per-round regret of Algorithm \ref{alg:expwt} measured after $t$ timesteps.

\emph{Weighted $k$-Means:} In the experiments for weighted $k$-Means, we consider random instances of $n=10$ points and where $k=3$ defined as follows.
The points lie in $R^2$ and are picked independently from $k$ gaussians.
The weights are picked uniformly and independently at random in the interval $[0,1]$.
We run the family of greedy heuristics described in section \ref{sect:examples}, with $\rho \in [0,1]$ and for $T=1000$ and
again performed $100$ rounds of the algorithm.
Again, Figure \ref{fig:kmeans} shows the mean per-round regret of Algorithm \ref{alg:expwt} measured after $t$ timesteps.

Figures \ref{fig:knap},\ref{fig:mwis} and \ref{fig:kmeans} show that the per-round regret of Algorithm \ref{alg:expwt} decreases
quickly and that the algorithm has significantly better performances on random instances than in the worst-case scenario.
In all the cases, the variance introduced by the internal randomness of the No-Regret Algorithm is very small.

\bibliography{all-refs}
\bibliographystyle{plainnat}

\end{document}